\documentclass[letterpaper,10pt,conference]{ieeeconf}
\IEEEoverridecommandlockouts                              
\usepackage{graphicx,amsmath,amssymb}
\usepackage{algorithmic,algorithm}
\usepackage{psfrag}
\usepackage{multirow}
\usepackage{url}
\usepackage{xcolor}

\newcommand{\A}{\ensuremath{\mathbf{A}}}
\newcommand{\B}{\ensuremath{\mathbf{B}}}

\newcommand{\F}{\ensuremath{\mathbf{F}}}
\newcommand{\G}{\ensuremath{\mathbf{G}}}

\newcommand{\I}{\ensuremath{\mathbf{I}}}

\newcommand{\OO}{\ensuremath{\mathbf{O}}}
\newcommand{\PP}{\ensuremath{\mathbf{P}}}
\newcommand{\Q}{\ensuremath{\mathbf{Q}}}

\newcommand{\U}{\ensuremath{\mathbf{U}}}
\newcommand{\V}{\ensuremath{\mathbf{V}}}
\newcommand{\W}{\ensuremath{\mathbf{W}}}
\newcommand{\X}{\ensuremath{\mathbf{X}}}
\newcommand{\Y}{\ensuremath{\mathbf{Y}}}

\newcommand{\f}{\ensuremath{\mathbf{f}}}
\newcommand{\g}{\ensuremath{\mathbf{g}}}

\newcommand{\uu}{\ensuremath{\mathbf{u}}}
\newcommand{\vv}{\ensuremath{\mathbf{v}}}

\newcommand{\x}{\ensuremath{\mathbf{x}}}
\newcommand{\y}{\ensuremath{\mathbf{y}}}

\newcommand{\1}{\ensuremath{\mathbf{1}}}

\newcommand{\bSigma}{\ensuremath{\boldsymbol{\Sigma}}}

\newcommand{\bbR}{\ensuremath{\mathbb{R}}}

\newcommand{\calO}{\ensuremath{\mathcal{O}}}

\newcommand{\abs}[1]{\left\lvert#1\right\rvert}
\newcommand{\norm}[1]{\left\lVert#1\right\rVert}

\newcommand{\caja}[4][1]{{%
    \renewcommand{\arraystretch}{#1}%
    \begin{tabular}[#2]{@{}#3@{}}%
      #4%
    \end{tabular}%
    }}

\providecommand{\keywords}[1]{\vspace*{0.5\baselineskip}\par\noindent\textbf{Keywords:} #1}

{%
\begin{list}{#1}{
\vspace{-\topsep}
\vspace{-\partopsep}
\setlength{\itemindent}{0cm}
\setlength{\rightmargin}{0cm}
\setlength{\listparindent}{0cm}
\settowidth{\labelwidth}{#1}
\setlength{\leftmargin}{\labelwidth}
\addtolength{\leftmargin}{\labelsep}
\setlength{\itemsep}{0cm}
}%
}%
{%
\end{list}
\vspace{-\topsep}
\vspace{-\partopsep}
}

{\begin{enumerate}%
\renewcommand{\theenumi}{\roman{enumi}}%
}%
{\end{enumerate}}

\newcommand{\traceop}{\operatorname{tr}}
\newcommand{\trace}[1]{\ensuremath{\traceop\left(#1\right)}}

\newcommand{\Sfg}{\bSigma_{fg}}
\newcommand{\Sff}{\bSigma_{ff}}
\newcommand{\Sgg}{\bSigma_{gg}}
\newcommand{\Tfg}{\tilde{\bSigma}_{fg}}
\newcommand{\tU}{\tilde{\U}}
\newcommand{\tV}{\tilde{\V}}

\newcommand{\tf}{\tilde{\f}}
\newcommand{\tg}{\tilde{\g}}
\newcommand{\Stf}{\bSigma_{\tf\tf}}
\newcommand{\Stg}{\bSigma_{\tg\tg}}

\newtheorem{theorem}{\textbf{Theorem}}

\graphicspath{{./}}

\title{\LARGE \bf
Stochastic Optimization for Deep CCA 
\\ via Nonlinear Orthogonal Iterations
}
\author{Weiran Wang$^1$, Raman Arora$^2$, Karen Livescu$^1$ and Nathan Srebro$^1$\\
$^1$Toyota Technological Institute at Chicago \hspace{1em} 6045 S. Kenwood Ave., Chicago, IL 60637 \\
$^2$Johns Hopkins University \hspace{1em} 3400 N. Charles St., Baltimore, MD 21218 \\
 Email: \texttt{$^1$\{weiranwang,klivescu,nati\}@ttic.edu, $^2$arora@cs.jhu.edu}
\thanks{This research was supported by the NSF grants IIS-1546482 and IIS-1321015. The opinions expressed in this work are those of the authors and do not necessarily reflect the views of the funding agency.}}

\begin{document}

\maketitle
\thispagestyle{empty}
\pagestyle{empty}

\begin{abstract}


Deep CCA is a recently proposed deep neural network extension to the traditional canonical correlation analysis (CCA), and has been 
successful for multi-view representation learning in several domains.  However, stochastic optimization of the deep CCA objective is not straightforward, because it does not decouple over training examples. Previous optimizers for deep CCA are either 
batch-based algorithms or stochastic optimization using large minibatches, which can have high memory consumption. In this paper, we tackle the problem of stochastic optimization for deep CCA with small minibatches, based on an iterative solution to the CCA objective, and show that we can achieve as good performance as previous optimizers and thus alleviate the memory requirement.  

\end{abstract}

\section{Introduction}
\label{s:intro}

Stochastic gradient descent (SGD) is a fundamental and popular optimization method for machine learning problems~\cite{Bottou91a,Lecun_98b,Bottou04a,Zhang04b,Bertsek11a}. SGD is particularly well-suited for large-scale machine learning problems because it is extremely simple and easy to implement, it often achieves better generalization (test) performance (which is the focus of machine learning research) than sophisticated batch algorithms, and it usually achieves large error reduction very quickly in a small number of passes over the training set~\cite{BottouBousquet08a}. One intuitive explanation for the empirical success of stochastic gradient descent for large data is that it makes better use of data redundancy, with an extreme example given by \cite{Lecun_98b}:  If the training set consists of $10$ copies of the same set of examples, then computing an estimate of the gradient over one single copy is $10$ times more efficient than computing the full gradient over the entire training set, while achieving the same optimization progress in the following gradient descent step.

At the same time, ``multi-view'' data are becoming increasingly available, and methods based on canonical correlation analysis (CCA)~\cite{Hotell36a} that use such data to learn representations (features) form an active research area. The views can be multiple measurement modalities, such as simultaneously recorded audio + video~\cite{Kidron_05a,Chaudh_09a}, audio + articulation~\cite{AroraLivesc13a}, images + text~\cite{Hardoon_04a,SocherLi10a,Hodosh_13a}, or parallel text in two languages~\cite{Vinokour_03a,Haghig_08a,Chandar_14a,FaruquiDyer14a,Lu_15a}, but may also be different information extracted from the same source, such as words + context~\cite{Pennin_14a} or document text + text of inbound hyperlinks~\cite{BickelScheff04a}. The presence of multiple information sources presents an opportunity to learn better representations (features) by analyzing multiple views simultaneously. Among various multi-view learning approaches, the recently proposed deep canonical correlation analysis \cite{Andrew_13a}, which extends traditional CCA with deep neural networks (DNNs), has been shown to be advantageous over previous methods in several domains \cite{Wang_15a,Wang_15b,YanMikolaj15a}, and scales to large data better than its nonparametric counterpart kernel CCA~\cite{LaiFyfe00a,BachJordan02a,Hardoon_04a}.

In contrast with most DNN-based methods, the objective of deep CCA couples together all of the training examples due to its whitening constraint, making stochastic optimization challenging. Previous optimizers for this model are 
batch-based, e.g., limited-memory BFGS (L-BFGS) \cite{Nocedal80a} as in \cite{Andrew_13a}, or stochastic optimization with large minibatches~\cite{Wang_15a}, because it is difficult to obtain an accurate estimate of the gradient with a small subset of the training examples (again due to the whitening constraint). As a result, these approaches have high memory complexity and may not be practical for large DNN models with hundreds of millions of weight parameters (common with web-scale data~\cite{Dean_12a}), or if one would like to run the training procedure on GPUs which are equipped with faster but smaller (more expensive) memory than CPUs.  In such cases there is not enough memory to save all intermediate hidden activations of the batch/large minibatch used in error backpropagation. 

In this paper, we tackle this problem with two key ideas. First, we reformulate the CCA solution with orthogonal iterations, and embed the DNN parameter training in the orthogonal iterations with a nonlinear least squares regression objective, which naturally decouples over training examples. Second, we use adaptive estimates of the covariances used by the CCA whitening constraints and carry out whitening \emph{only} for the minibatch used at each step to obtain training signals for the DNNs. This results in a stochastic optimization algorithm that can operate on small minibatches and thus consume little memory. Empirically, the new stochastic optimization algorithm performs as well as previous optimizers in terms of convergence speed, even when using small minibatches with which the previous stochastic approach makes no training progress.

In the following sections, we briefly introduce deep CCA and discuss the difficulties in training it (Section~\ref{s:dcca}); motivate and propose our new algorithm (Section~\ref{s:algorithm}); describe related work (Section~\ref{s:related}); and present experimental results comparing different optimizers (Section~\ref{s:experiments}).

\section{Deep CCA}
\label{s:dcca}

\noindent\textbf{Notation} In the multi-view feature learning setting, we have access to paired observations from two views, denoted ${(\x_1,\y_1),\dots,(\x_N,\y_N)}$, where $N$ is the training set size, $\x_i\in \bbR^{D_x}$ and $\y_i\in \bbR^{D_y}$ for $i=1,\dots,N$. We also denote the data matrices for View 1 and View 2 $\X=[\x_1,\dots,\x_N]$ and $\Y=[\y_1,\dots,\y_N]$, respectively. We use bold-face letters, e.g.~$\f$, to denote mappings implemented by DNNs, with a corresponding set of learnable parameters, denoted, e.g., $\W_\f$. The dimensionality of the 
learned features is denoted $L$.

\begin{figure}
  \centering
  \psfrag{x}[][]{$\x$}
  \psfrag{y}[][]{$\y$}
  \psfrag{v1}[][][0.8]{View 1}
  \psfrag{v2}[][][0.8][90]{View 2}
  \psfrag{U}[][]{$\U$}
  \psfrag{V}[][]{$\V$}
  \psfrag{f}[][]{$\f$}
  \psfrag{g}[][]{$\g$}
  \includegraphics[width=0.55\linewidth]{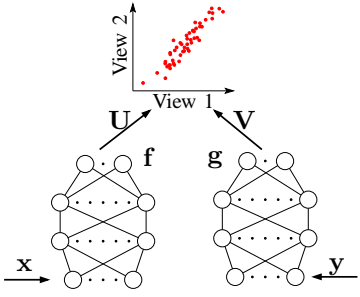}
  \caption{Schematic diagram of deep canonical correlation analysis.}
  \label{f:dcca}
\end{figure}

Deep CCA (DCCA)~\cite{Andrew_13a} extends (linear) CCA~\cite{Hotell36a} by extracting $d_x$- and $d_y$-dimensional nonlinear features with two DNNs $\f$ and $\g$ for views 1 and 2 respectively, such that the canonical correlation (measured by CCA) between the DNN outputs is maximized, as illustrated in Fig.~\ref{f:dcca}. The goal of the final CCA is to find $L \le \min(d_x,d_y)$ pairs of linear projection vectors $\U \in \bbR^{d_x \times L} $ and $\V \in \bbR^{d_y \times L}$ such that the projections of each view (a.k.a.~canonical variables,~\cite{Hotell36a}) are maximally correlated with their counterparts in the other view, constrained such that the dimensions in the representation are uncorrelated with each other.
Formally, the DCCA objective can be written as\footnote{In this paper, we use the scaled covariance matrices (scaled by $N$) so that the dimensions of the projection are orthonormal and comply with the custom of orthogonal iterations.}
\begin{gather} \label{e:dcca}
  \max_{\W_\f, \W_\g, \U, \V} \quad \trace{\U^\top \F  \G^\top \V} \\
  \text{s.t.} \quad \U^\top \F \F^\top \U = \V^\top \G \G^\top \V = \I, \nonumber
\end{gather}
where $\F=\f(\X)=[\f(\x_1),\dots,\f(\x_N)] \in \bbR^{d_x \times N}$ and $\G=\g(\Y)=[\g(\y_1),\dots,\g(\y_N)] \in \bbR^{d_y \times N}$. We assume that $\F$ and $\G$ are centered at the origin for notational simplicity; if they are not, we can center them as a pre-processing operation. Notice that if we use the original input data without further feature extraction, i.e.~$\F=\X$ and $\G=\Y$, then we recover the CCA objective. 
In DCCA, the final features (projections) are 
\begin{gather}\label{e:concat}
  \tf(\x)=\U^\top \f(\x) \qquad \text{and} \qquad \tg(\y)=\V^\top \g(\y).
\end{gather}
We observe that the last CCA step with linear projection mappings $\U$ and $\V$ can be considered as adding a linear layer on top of the feature extraction networks $\f$ and $\g$ respectively. In the following, we sometimes refer to the concatenated networks $\tf$ and $\tg$ as defined in \eqref{e:concat}, with $\W_{\tf}=\{\W_\f,\U\}$ and $\W_{\tg}=\{\W_\g,\V\}$.
\footnote{In principle there is no need for the final linear layer; we could define DCCA such that the correlation objective and constraints are imposed on the final nonlinear layer.  However, the linearity of the final layer is crucial for algorithmic implementations such as ours.}

Let $\Sfg= \F \G^\top$, $\Sff=\F \F^\top$ and $\Sgg=\G \G^\top$ be the (scaled) cross- and auto-covariance matrices of the feature-mapped data in the two views. It is well-known that, when $\f$ and $\g$ are fixed, the last CCA step in \eqref{e:dcca} has a closed form solution as follows. Define $\Tfg=\Sff^{-\frac{1}{2}} \Sfg \Sgg^{-\frac{1}{2}}$, and let $\Tfg=\tU \Lambda \tV^\top $ be its rank-L singular value decomposition (SVD), where $\Lambda$ contains the singular values $\sigma_1 \ge \dots \ge \sigma_L \ge 0$ on its diagonal. Then the optimum of \eqref{e:dcca} is achieved by $(\U,\V)=(\Sff^{-\frac{1}{2}} \tU, \Sgg^{-\frac{1}{2}} \tV )$, and the optimal objective value (the total canonical correlation) is $\sum_{j=1}^L \sigma_j$. By switching $\max(\cdot)$ with $- \min -(\cdot)$, and adding $1/2$ times the constraints, it is straightforward to show that \eqref{e:dcca} is equivalent to the following:
\begin{gather}\label{e:dcca2}
  \min_{\W_\f, \W_\g, \U, \V} \quad \frac{1}{2} \norm{\U^\top \F -  \V^\top \G}^2_F \\
  \qquad \text{s.t.} \quad  (\U^\top \F) (\U^\top \F)^\top = (\V^\top \G) (\V^\top \G)^\top = \I. \nonumber
\end{gather}
In other words, CCA minimizes the squared difference between the projections of the two views, subject to the whitening constraints. This alternative formulation of CCA will also shed light on our proposed algorithm for DCCA. 

The DCCA objective \eqref{e:dcca} differs from typical DNN regression or classification training objectives.  Typically, the objectives are unconstrained and can be written as the expectation (or sum) of error functions (e.g., squared loss or cross entropy) incurred at each training example.  This property naturally suggests stochastic gradient descent (SGD) for optimization, where one iteratively generates random unbiased estimates of the gradient based on one or a few training examples (a minibatch) and takes a small step in the opposite direction. However, the objective in \eqref{e:dcca} can not be written as an unconstrained sum of errors. The difficulty lies in the fact that the training examples are coupled through the auto-covariance matrices (in the constraints), which can not be reliably estimated with only a small amount of data. 

When introducing deep CCA, \cite{Andrew_13a} used the L-BFGS algorithm for optimization. To compute the gradients of the objective with respect to $(\W_\f,\W_\g)$, one first computes the gradients\footnote{Technically we are computing subgradients as the ``sum of singular values'' (trace norm) is not a differentiable function of the matrix.} with respect to $(\F,\G)$ as
\begin{align}\label{e:gradient}
  \frac{\partial \sum_{j=1}^L \sigma_j} {\partial \F} &= 2\Delta_{ff} \F + \Delta_{fg} \G, \\
  \text{with}\qquad \Delta_{ff} & = -\frac{1}{2} \bSigma_{ff}^{-1/2} \tilde{\U} \Lambda \tilde{\U}^\top \bSigma_{ff}^{-1/2} \nonumber \\
  \Delta_{fg} & = \bSigma_{ff}^{-1/2} \tilde{\U} \tilde{\V}^\top \bSigma_{gg}^{-1/2} \nonumber
\end{align}
where $\Tfg=\tU\Lambda\tV^\top$ is the SVD of $\Tfg$ as in the closed-form solution to CCA, and $\partial \sum_{j=1}^L \sigma_j / \partial \G$ has an analogous expression. One can then compute the gradients with respect to $\W_\f$ and $\W_\g$ via the standard backpropagation procedure~\cite{Rumelh_86c}. From the gradient formulas, it is clear that the key to optimizing DCCA is the SVD of $\Tfg$; various nonlinear optimization techniques can be used here once the gradient is computed. In practice, however, batch optimization is undesirable for applications with large training sets or large DNN architectures, as each gradient step computed on the entire training set can be 
expensive in both memory and time.

Later, it was observed by \cite{Wang_15a} that stochastic optimization still works well even for the DCCA objective, as long as larger minibatches are used to estimate the covariances and $\Tfg$ when computing the gradient with \eqref{e:gradient}. More precisely, the authors find that learning plateaus at a poor objective value if the minibatch is too small, but fast convergence and better generalization than batch algorithms can be obtained once the minibatch size is larger than some threshold, presumably because a large minibatch contains enough information to estimate the covariances and therefore the gradient accurately enough (the threshold of minibatch size varies for different datasets because they have different levels of data redundancy). 
Theoretically, the necessity of using large minibatches in this approach 
can also be established. Let the empirical estimate of $\Tfg$ using a minibatch of $n$ samples be $\hat{\bSigma}_{fg}^{(n)}$. 
It can be shown that the expectation of $\hat{\bSigma}_{fg}^{(n)}$ does not equal the true $\Tfg$ computed using the entire dataset, mainly due to the nonlinearities in the matrix inversion and multiplication operations in computing $\Tfg$, and the nonlinearity in the ``sum of singular values'' 
(trace norm) of $\Tfg$; moreover, the spectral norm of the error $\norm{\hat{\bSigma}_{fg}^{(n)} - \Tfg}$ decays slowly 
as $\frac{1}{\sqrt{n}}$. Consequently, the gradient estimated on a minibatch using \eqref{e:gradient} does not equal the true gradient of the objective in expectation, indicating that the stochastic approach of \cite{Wang_15a} does not qualify as a stochastic gradient descent method for the DCCA objective.

\section{Our algorithm}
\label{s:algorithm}

\subsection{An iterative solution to linear CCA}

\begin{algorithm}[t]
  \caption{CCA projections via alternating least squares.}
  \label{alg:cca-iterative}
  \renewcommand{\algorithmicrequire}{\textbf{Input:}}
  \renewcommand{\algorithmicensure}{\textbf{Output:}}
  \begin{algorithmic}
    \REQUIRE Data matrices $\F\in \bbR^{d_x \times N}$, $\G\in \bbR^{d_y \times N}$. Initialization $\tU_0\in \bbR^{d_x\times L}$ s.t.  $\tU_0^\top \tU_0=\I$.
    \STATE $\A_0 \leftarrow \tU_0^\top \Sff^{-\frac{1}{2}} \F$
    \FOR{$t=1,2,\dots,T$}
    \STATE $\B_t \leftarrow \A_{t-1} \G^\top \left( \G \G^\top \right)^{-1} \G$
    \STATE $\B_{t} \leftarrow \left(\B_{t}\B_{t}^\top \right)^{-\frac{1}{2}} \B_t$
    \STATE $\A_t \leftarrow \B_{t} \F^\top \left( \F \F^\top \right)^{-1} \F$
    \STATE $\A_{t} \leftarrow \left(\A_{t}\A_{t}^\top \right)^{-\frac{1}{2}} \A_t$ 
    \ENDFOR  
    \ENSURE $\A_{T}$/$\B_{T}$ are the CCA projections of view 1/2.
  \end{algorithmic}
\end{algorithm}

Our solution to \eqref{e:dcca} is inspired by the iterative solution for finding the linear CCA projections $(\U^\top \F, \V^\top \G)$ for inputs $(\F, \G)$, as shown in Algorithm~\ref{alg:cca-iterative}. This algorithm computes the top-$L$ singular vectors $(\tU,\tV)$ of $\Tfg$ via orthogonal iterations \cite{GolubLoan96a}. An essentially identical algorithm (named \emph{alternating least squares} for reasons that will soon become evident) appears in \cite[Algorithm 5.2]{GolubZha95a} and according to the authors the idea goes back to J. Von Neumann. A similar algorithm is also recently used by \cite[Algorithm~1]{LuFoster14a} for large scale linear CCA with high-dimensional sparse inputs, although their algorithm does not implement the whitening operations $\A_{t} \leftarrow \left(\A_{t}\A_{t}^\top \right)^{-\frac{1}{2}} \A_t$ and $\B_{t} \leftarrow \left(\B_{t}\B_{t}^\top \right)^{-\frac{1}{2}} \B_t$ or they use the QR decomposition instead. The convergence of Algorithm~\ref{alg:cca-iterative} is characterized by the following theorem, which parallels \cite[Theorem~1]{LuFoster14a}.

\begin{theorem}
  Let the singular values of $\Tfg$ be
  \begin{gather*}
    \sigma_1 \ge \dots \ge \sigma_L > \sigma_{L+1} \ge \dots \ge \sigma_{\min(d_x,d_y)}
  \end{gather*}
  and 
suppose $\tU_0^\top \tU$ is nonsingular. Then the output $(\A_T,\B_T)$ of Algorithm~\ref{alg:cca-iterative} converges to the CCA projections as $T\rightarrow \infty$.
\end{theorem}
\begin{proof}
  We focus on showing that $\A_T$ converges to the view 1 projection; the proof for $\B_T$ is similar.

  First recall that $\Tfg=\tU \Lambda \tV^\top$ is the rank-$L$ SVD of $\Sff^{-\frac{1}{2}} \Sfg \Sgg^{-\frac{1}{2}}$, and thus $\tU$ contains the top-$L$ eigenvectors of $\Tfg \Tfg^\top = \tU \Lambda^2 \tU^\top$.

  Since the operation $\left(\A \A^\top\right)^{-\frac{1}{2}}\A$ extracts an orthonormal basis of the row space of $\A$, at iteration $t$ we can write 
  \begin{align*}
    \A_{t-1} \G^\top \left( \G \G^\top \right)^{-1} \G & = \PP_t \B_t\\
    \B_{t} \F^\top \left( \F \F^\top \right)^{-1} \F & = \Q_t \A_t
  \end{align*}
  where $\PP_t \in \bbR^{L\times L}$ and $\Q_t \in \bbR^{L\times L}$ are nonsingular coefficient matrices (as the initialization $\tU_0$ is nonsingular) for representing the left-hand side matrices in their row space basis. Combining the above two equations gives the following recursion at iteration $t$:
  \begin{gather*}
    \A_{t-1} \G^\top \left( \G \G^\top \right)^{-1} \G \F^\top \left( \F \F^\top \right)^{-1} \F = \PP_t \Q_t \A_t.
  \end{gather*}
  By induction, it can be shown that by the end of iteration $t$ we have
  \begin{multline*}
    \A_0 \left( \G^\top \left( \G \G^\top \right)^{-1} \G \F^\top \left( \F \F^\top \right)^{-1} \F \right)^t = \OO_t \A_t.
  \end{multline*}
where $\OO_t=\PP_1 \Q_1 \dots \PP_t \Q_t \in \bbR^{L\times L}$ is nonsingular.
  Plugging in the definition of $\A_0$, this equation reduces to
  \begin{gather} \label{e:orth-iteration}
    \tU_0^\top \left(\Tfg \Tfg^\top \right)^t \Sff^{-\frac{1}{2}} \F = \OO_t \A_t.
  \end{gather}
It is then clear that $\A_t$ can be written as 
\begin{gather*}
\A_t = \tU_t^\top \Sff^{-\frac{1}{2}} \F
\end{gather*}
 with 
\begin{gather*}
\tU_t = \left(\Tfg \Tfg^\top \right)^t \tU_0 \OO_t^{-1} \; \in \bbR^{d_x\times L}.
\end{gather*} 
And since $\A_t$ has orthonormal rows, we have
\begin{gather*}
\I=\A_t \A_t^\top = \tU_t^\top \Sff^{-\frac{1}{2}} (\F \F^\top) \Sff^{-\frac{1}{2}}  \tU_t = \tU_t^\top \tU_t,
\end{gather*}
indicating that $\tU_t$ has orthonormal columns.

As a result, we consider the algorithm as working implicitly in the space of $\{ \tU_t\in \bbR^{d_x\times L}, t=0,\dots,T\}$, and have
  \begin{gather} \label{e:orth-iteration}
     (\Tfg \Tfg^\top)^T \tU_0  = \OO_T \tU_T.
  \end{gather}
Following the argument of~\cite[Theorem~8.2.2]{GolubLoan96a}) for orthogonal iterations, under the assumptions of our theorem, the column space of $\tU_T$ converges to that of $\tU$, the top-$L$ eigenvectors of $\Tfg \Tfg^\top$, with a linear convergence rate depending on the ratio $\sigma_{L+1}/\sigma_L$. In view of the relationship between $\tU_T$ and $\A_t$, we conclude that $\A_T$ converges to the view 1 CCA projection as $T\rightarrow \infty$.
\end{proof}

It is interesting to note that, besides the whitening operations $\left(\A_{t}\A_{t}^\top \right)^{-\frac{1}{2}} \A_t$, the other basic operations in each iteration of Algorithm~\ref{alg:cca-iterative} are of the form 
\begin{gather}\label{e:lsq}
  \A_t \leftarrow \B_{t} \F^\top \left( \F \F^\top \right)^{-1} \F
\end{gather}
which is solving a linear least squares (regression) problem with input $\F$ and target output $\B_{t}$ satisfying $\B_{t}\B_{t}^\top=\I$, i.e.,
\begin{gather*}
  \min_{\U_t} \quad \norm{\U_t^\top \F - \B_t}_F^2. 
\end{gather*}
By setting the gradient of this unconstrained objective to zero, we obtain $\U_t=(\F\F^\top)^{-1} \F \B_t^\top$ and so the optimal projection $\U_t^\top \F$ coincides with the update \eqref{e:lsq}.

For \cite{LuFoster14a}, the advantage of the alternating least squares formulation over the exact solution to CCA is that it does not need to form the high-dimensional (nonsparse) matrix $\Tfg$; instead it directly operates on the projections, which are much smaller in size, and one can solve the least squares problems using iterative algorithms that require only sparse matrix-vector multiplications.

\subsection{Extension to DCCA}

Our intuition for adapting Algorithm~\ref{alg:cca-iterative} to DCCA is as follows. During DCCA optimization, the DNN weights $(\W_\f,\W_\g)$ are updated frequently and thus the outputs $\left( \f(\X),\g(\Y) \right)$, which are also the inputs to the last CCA step, also change upon each weight update. Therefore, the last CCA step needs to adapt to the fast evolving input data distribution. On the other hand, if we are updating the CCA weights $(\U,\V)$ based on a small minibatch of data (as happens in stochastic optimization), it is intuitively wasteful to solve $(\U,\V)$ to optimality rather than to make a simple update based on the minibatch. Moreover, the objective of this ``simple update''  can be used to derive a gradient estimate for $(\W_\f, \W_\g)$.

In view of Algorithm~\ref{alg:cca-iterative}, it is a natural choice to embed the optimization of $(\f, \g)$ into the iterative solution to linear CCA. Instead of solving the regression problem $\F \rightarrow \B_{t}$ exactly with $\A_t \leftarrow \B_{t} \F^\top \left( \F \F^\top \right)^{-1} \F$, we try to solve the problem $\X \rightarrow \B_{t}$ on a minibatch with a gradient descent step on $(\W_\f, \U)$ jointly (recall $\F=\f(\X)$ is a function of $\W_\f$). Notice that this regression objective is unconstrained and decouples over training samples, so an unbiased gradient estimate for this problem can be easily derived through standard backpropagation using minibatches (however, this gradient estimate may not be unbiased for the original DCCA objective; see discussion in Section~\ref{s:related}).

The less trivial part of Algorithm~\ref{alg:cca-iterative} to implement in DCCA is the whitening operation $\left(\A_{t}\A_{t}^\top \right)^{-\frac{1}{2}} \A_t$, which needs $\A_t\in\bbR^{L\times N}$, the projections of all training samples. We would like to avoid the exact computation of $\A_t$ as it requires feeding forward the entire training set $\X$ with the updated $\W_{\tf}$, and the computational cost of this operation is as high as (half of) the cost of evaluating the batch gradient (the latter requires both the forward and backward passes). We bypass this difficulty by noting that the only portion of $\A_{t}$ needed is the updated projection of the minibatch used in the subsequent view 2 regression problem $\X \rightarrow \A_{t}$ (corresponding to the step $\B_{t+1} \leftarrow \A_t \G^\top \left( \G \G^\top \right)^{-1} \G$ in Algorithm~\ref{alg:cca-iterative}). Therefore, if we have an estimate of the covariance $\Stf^t:=\A_{t}\A_{t}^\top$ without feeding forward the entire training set, we can estimate the updated projection for this minibatch only. Specifically, we estimate this quantity by\footnote{We add a small value $\epsilon>0$ to the diagonal of the covariance estimates in our implementation for numerical stability.}
\begin{gather}\label{e:memory}
  \Stf^{t} \leftarrow \rho\Stf^{t-1} + (1-\rho) \frac{N}{\abs{b}} \tf(\X_b)\tf(\X_b)^\top,
\end{gather}
where $\rho\in[0,1]$, $\X_b$ denotes a minibatch of data with index set $b$, and $\abs{b}$ denotes the size (number of samples) of this minibatch. The time constant $\rho$ controls how much the previous covariance estimate is kept in the update; a larger $\rho$ indicates forgetting the ``memory'' more slowly. Assuming that the parameters do not change much from time $t-1$ to $t$, then $\Stf^{t-1}$ will be close to $\Stf^{t}$, and incorporating it helps to reduce the variance from the term $\tf(\X_b)\tf(\X_b)^\top$ when $\abs{b}\ll N$.
The update in \eqref{e:memory} has a form similar to that of the widely used momentum technique in the optimization~\cite{Polyak64a} and neural network literature~\cite{Sutskev_13a,Schaul_13a}, and is also used by \cite{Brand06a,SantosMilidiu10a,Yger_12a} for online subspace tracking and anomaly detection. We note that the memory cost of $\Stf^{t} \in \bbR^{L\times L}$ is small as we look for low-dimensional projections (small $L$) in practice. These advantages validate our choice of whitening operations over the more commonly used QR decomposition used by \cite{LuFoster14a}. 

\begin{algorithm}[t]
  \caption{Nonlinear orthogonal iterations (NOI) for DCCA.}
  \label{alg:dcca}
  \renewcommand{\algorithmicrequire}{\textbf{Input:}}
  \renewcommand{\algorithmicensure}{\textbf{Output:}}
  \begin{algorithmic}
    \REQUIRE Data matrix $\X\in \bbR^{D_x \times N}$, $\Y\in \bbR^{D_y \times N}$. Initialization $(\W_{\tf}, \W_{\tg})$, time constant $\rho$, learning rate $\eta$.
    \STATE Randomly choose a minibatch $(\X_{b_0},\Y_{b_0})$ 
    \STATE $\Stf \leftarrow \frac{N}{b_0}\sum_{i\in b_0} \tf(\x_i)\tf(\x_i)^\top$, 
    \STATE $\Stg \leftarrow \frac{N}{b_0}\sum_{i\in b_0} \tg(\y_i)\tg(\y_i)^\top$
    \FOR{$t=1,2,\dots,T$}
    \STATE Randomly choose a minibatch $(\X_{b_t},\Y_{b_t})$
    \STATE $\Stf \leftarrow \rho \Stf + (1-\rho) \frac{N}{\abs{b_t}}\sum_{i\in b_t} \tf(\x_i)\tf(\x_i)^\top$
    \STATE $\Stg \leftarrow \rho \Stg + (1-\rho) \frac{N}{\abs{b_t}}\sum_{i\in b_t} \tg(\y_i)\tg(\y_i)^\top$
    \STATE Compute the gradient $\partial \W_{\tf}$ of the objective 
    \begin{gather*}
      \min_{\W_{\tf}}\; \frac{1}{\abs{b_t}} \sum_{i\in b_t} \norm{\tf(\x_i) - \Stg^{-\frac{1}{2}}\tg(\y_i) }^2
    \end{gather*}
    \STATE Compute the gradient $\partial \W_{\tg}$ of the objective 
    \begin{gather*}
      \min_{\W_{\tg}}\; \frac{1}{\abs{b_t}} \sum_{i\in b_t} \norm{\tg(\y_i) - \Stf^{-\frac{1}{2}}\tf(\x_i) }^2
    \end{gather*}
    \STATE $\W_{\tf} \leftarrow \W_{\tf} - \eta \partial \W_{\tf}$, $\W_{\tg} \leftarrow \W_{\tg} - \eta \partial \W_{\tg}$.
    \ENDFOR  
    \ENSURE The updated $(\W_{\tf}, \W_{\tg})$.
  \end{algorithmic}
\end{algorithm}

We give the resulting nonlinear orthogonal iterations procedure (NOI) for DCCA in Algorithm~\ref{alg:dcca}. Now adaptive whitening is used to obtain suitable target outputs of the regression problems for computing derivatives $(\partial \W_{\tf}, \partial \W_{\tg})$, and we no longer maintain the whitened projections of the entire training set at each iteration.
Therefore, by the end of the algorithm, $(\tf(\X),\tg(\Y))$ may not satisfy the whitening constraints of \eqref{e:dcca}. One may use an additional CCA step on $(\tf(\X),\tg(\Y))$ to obtain a feasible solution of the original problem if desired, and this amounts to linear transforms in $\bbR^L$ which do not change the canonical correlations between the projections for both the training and test sets.
In practice, we adaptively estimate the mean of $\tf(\X)$ and $\tg(\Y)$ with an update formula similar to that of \eqref{e:memory} and center the samples accordingly before estimating the covariances and computing the target outputs. We also use momentum in the stochastic gradient steps for the nonlinear least squares problems as is commonly used in the deep learning community \cite{Sutskev_13a}. Overall, Algorithm~\ref{alg:dcca} is intuitively quite simple:  It alternates between adaptive covariance estimation/whitening and stochastic gradient steps over (a stochastic version of) the least squares objectives, without any involved gradient computation.

\section{Related Work}
\label{s:related}

Stochastic (and online) optimization techniques for fundamental problems, such as principal component analysis and partial least squares, are of continuous research interest~\cite{Krasul69a,OjaKarhun85a,WarmutKuzmin08a,Arora_12a,Arora_13a,Mitliag_13a,Balsub_13a,Shamir15a}. However, as pointed out by \cite{Arora_12a}, the CCA objective is more challenging due to the whitening constraints. 

Recently, \cite{Yger_12a} proposed an adaptive CCA algorithm with efficient online updates based on matrix manifolds defined by the whitening constraints. However, the goal of their algorithm is anomaly detection rather than optimizing the canonical correlation objective for a given dataset. 
Based on the alternating least squares formulation of CCA (Algorithm~\ref{alg:cca-iterative}), \cite{LuFoster14a} propose an iterative solution of CCA for very high-dimensional and sparse input features, and the key idea is to solve the high dimensional least squares problems with randomized PCA and (batch) gradient descent. 

\begin{algorithm}[t]
  \caption{CCA via gradient descent over least squares. }
  \label{alg:cca-gd}
  \renewcommand{\algorithmicrequire}{\textbf{Input:}}
  \renewcommand{\algorithmicensure}{\textbf{Output:}}
  \begin{algorithmic}
    \REQUIRE Data matrix $\F\in \bbR^{d_x \times N}$, $\G\in \bbR^{d_y \times N}$. Initialization ${\uu}_0 \in \bbR^{d_x}$, ${\vv}_0 \in \bbR^{d_y}$. Learning rate $\eta$.
    \FOR{$t=1,2,\dots,T$}
    \STATE ${\uu}_t \leftarrow {\uu}_{t-1} - \eta \F (\F^\top {\uu}_{t-1} - \frac{1}{\norm{\vv_{t-1}^\top \G}} \G^\top {\vv}_{t-1})$
    \STATE ${\vv}_t \leftarrow {\vv}_{t-1} - \eta \G (\G^\top {\vv}_{t-1} - \frac{1}{\norm{\uu_{t-1}^\top \F}} \F^\top {\uu}_{t-1})$
    \ENDFOR  
    \STATE $\uu \leftarrow \frac{{\uu}_T}{\norm{\uu_T^\top \F}}$,$\quad$$\vv \leftarrow \frac{{\vv}_T}{\norm{\vv_T^\top \G}}$
    \ENSURE $\uu$/$\vv$ are the CCA directions of view 1/2.
  \end{algorithmic}
\end{algorithm}

Upon the submission of this paper, we have become aware of the very recent publication of \cite{Ma_15b}, which extends \cite{LuFoster14a} by solving the linear least squares problems with (stochastic) gradient descent. We notice that a specical case of our algorithm ($\rho=0$) is equivalent to theirs for linear CCA. To see this, we give the linear CCA version of our algorithm (for a one-dimensional projection, to be consistent with the notation of \cite{Ma_15b}) in Algorithm~\ref{alg:cca-gd}, where we take a batch gradient descent step over the least squares objectives in each iteration. This algorithm is equivalent to Algorithm~3 of \cite{Ma_15b}.\footnote{Although Algorithm~3 of \cite{Ma_15b} maintains two copies---the normalized and the unnormalized versions---of the weight parameters, we observe that the sole purpose of the normalized version in the intermediate iterations is to provide whitened target output for the least squares problems; our version of the algorithm eliminates this copy and the normalized version can be retrieved by a whitening step at the end.} Though intuitively very simple, the analysis of this algorithm is challenging.  
In~\cite{Ma_15b} it is shown that the solution to the CCA objective is a fixed point of this algorithm, but no global convergence property is given. We also notice that the gradients used in this algorithm are derived from the alternating least squares problems
\begin{gather*}
\min_{\uu}\; \norm{ \uu^\top \F - \frac{\vv^\top \G}{\norm{\vv^\top \G}} }_F^2 \text{\ and \ } \min_{\vv}\; \norm{ \vv^\top \G - \frac{\uu^\top \F}{\norm{\uu^\top \F}} }_F^2,
\end{gather*}
while the true CCA objective can be written as
\begin{gather*}
\min_{\uu,\vv}\; \norm{ \frac{\uu^\top \F}{\norm{\uu^\top \F}} - \frac{\vv^\top \G}{\norm{\vv^\top \G}}}_F^2.
\end{gather*}
This shows that Algorithm~3 is \emph{not} implementing gradient descent over the CCA objective.

When extending Algorithm~3 to stochastic optimization, we observe the key differences between their algorithm and ours as follows. 
Due to the evolving $(\W_\f,\W_\g)$, the last CCA step in the DCCA model is dealing with different $(\f(\X),\g(\Y))$ and covariance structures in different iterates, even though the original inputs $(\X,\Y)$ are the same; this motivates the adaptive estimate of covariances in \eqref{e:memory}. In the whitening steps of \cite{Ma_15b}, however, the covariances are estimated  using \emph{only} the current minibatch at each iterate, without consideration of the remaining training samples or previous estimates, which corresponds to $\rho\rightarrow 0$ in our estimate. \cite{Ma_15b} also 
suggests using a minibatch size of the order $\calO(L)$, the dimensionality of the covariance matrices to be estimated, in order to obtain a high-accuracy estimate for whitening. As we will show in the experiments, in both CCA and DCCA, it is important to incorporate the previous covariance estimates ($\rho\rightarrow 1$) at each step to reduce the variance, especially when small minibatches are used. Based on the above analysis for batch gradient descent, solving the least squares problem with stochastic gradient descent is \emph{not} implementing stochastic gradient descent over the CCA objective. Nonetheless, as shown in the experiments, this stochastic approach works remarkably well and can match the performance of batch optimization, for both linear and nonlinear CCA, and is thus worth careful analysis.

Finally, we remark that other possible approaches for solving \eqref{e:dcca} exist. Since the difficulty lies in the whitening constraints, one can relax the constraints and solve the Lagrangian formulation repeatedly with updated Lagrangian multipliers, as done by \cite{LaiFyfe00a}; or one can introduce auxiliary variables and apply the quadratic penalty method \cite{NocedalWright06a}, as done by \cite{CarreirWang14b}. The advantage of such approaches is that there exists no coupling of all training samples when optimizing the primal variables (the DNN weight parameters) and thus one can easily apply SGD there, but one also needs to deal with the Lagrange multipliers or to set a schedule for the quadratic penalty parameter (which is non-trivial) and alternately optimize over two sets of variables repeatedly in order to obtain a solution of the original constrained problem.

\begin{table}[t]
  \centering
  \caption{Statistics of two real-world datasets.}
  \label{t:datasets}
  \begin{tabular}{|c||c|c|c|}
    \hline
    dataset & training/tuning/test & $L$ & DNN architectures  \\ \hline
    JW11    &  30K/11K/9K  &  112 & \caja{c}{c}{273-1800-1800-112\\112-1200-1200-112} \\ \hline
    MNIST   &  50K/10K/10K &  50  & \caja{c}{c}{392-800-800-50\\392-800-800-50} \\
    \hline
  \end{tabular}
\end{table}

\section{Experiments}
\label{s:experiments}

\subsection{Experimental setup}
We now demonstrate the NOI algorithm on the two real-world datasets used by \cite{Andrew_13a} when introducing DCCA. The first dataset is a subset of the University of Wisconsin X-Ray Microbeam corpus ~\cite{Westbur94a}, which consists of simultaneously recorded acoustic and articulatory measurements during speech. Following \cite{Andrew_13a,Wang_15a}, the acoustic view inputs are 39D Mel-frequency cepstral coefficients and the articulatory view inputs are horizontal/vertical displacement of 8 pellets attached to different parts of the vocal tract, each then concatenated over a 7-frame context window, for speaker `JW11'. The second dataset consists of left/right halves of the images in the MNIST dataset~\cite{Lecun_98a}, and so the input of each view consists of $28\times 14$ grayscale images. We do not tune neural network architectures as it is out of the scope of this paper. Instead, we use DNN architectures similar to those used by \cite{Andrew_13a} with ReLU activations~\cite{NairHinton10a}, and we achieve better generalization performance with these architectures mainly due to better optimization. The statistics of each dataset and the chosen DNN architectures (widths of input layer-hidden layers-output layer) are given in Table~\ref{t:datasets}. The projection dimensionality $L$ is set to 112/50 for JW11/MNIST respectively as in \cite{Andrew_13a}; these are also the maximum possible total canonical correlations for the two datasets. 

We compare three optimization approaches: full batch optimization by L-BFGS~\cite{Andrew_13a}, using the implementation of \cite{Schmid12a} which includes a good line-search procedure; stochastic optimization with large minibatches~\cite{Wang_15a}, denoted STOL; and our algorithm, denoted NOI. We create training/tuning/test splits for each dataset and measure the total canonical correlations on the test sets (measured by linear CCA on the projections) for different optimization methods. Hyperparameters of each algorithm, including $\rho$ for NOI, minibatch size $n=\abs{b_1}=\abs{b_2},\dots$, learning rate $\eta$ and momentum $\mu$ for both STOL and NOI, are chosen by grid search on the tuning set. All methods use the same random initialization for DNN weight parameters. We set the maximum number of iterations to $300$ for L-BFGS and number of epochs (one pass over the training set) to $50$ for STOL and NOI.

\begin{table*}[!t]\centering
  \caption{Total test set canonical correlation obtained by different algorithms.} 
  \label{t:corr}
  \begin{tabular}{|c|c|c|c|c|c|c|c|}\hline
    \multirow{2}{*}{dataset} & \multirow{2}{*}{L-BFGS} & \multicolumn{2}{|c}{STOL} & \multicolumn{4}{|c|}{NOI} \\ \cline{3-8}
    && $n=100$ & $n=500$ &  $n=10$ & $n=20$ & $n=50$ & $n=100$ \\ \hline
    JW11    & 78.7   & 33.0         & 86.7         & 83.6       & 86.9       & 87.9       & 89.1 \\ \hline
    MNIST   & 47.0   & 26.1         & 47.0         & 45.9       & 46.4       & 46.4       & 46.4 \\ \hline
  \end{tabular}
\end{table*}

\begin{figure}[t]
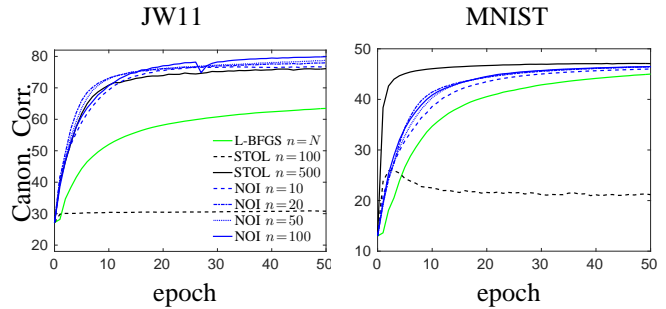

  \centering
  \begin{tabular}{@{}c@{\hspace{0.03\linewidth}}c@{}}
    JW11 & MNIST \\[.5ex]
    \psfrag{corr}[][]{Canon. Corr.}
    \psfrag{iteration}[t][]{epoch}
\psfrag{LBFGS n=N}[l][l][0.52]{L-BFGS $n\!=\!N$}
\psfrag{STOL n=100}[l][l][0.55]{STOL $n\!=\!100$}
\psfrag{STOL n=500}[l][l][0.55]{STOL $n\!=\!500$}
\psfrag{NOI n=10}[l][l][0.55]{NOI $n\!=\!10$}
\psfrag{NOI n=20}[l][l][0.55]{NOI $n\!=\!20$}
\psfrag{NOI n=50}[l][l][0.55]{NOI $n\!=\!50$}
\psfrag{NOI n=100}[l][l][0.55]{NOI $n\!=\!100$}
    \includegraphics[width=0.50\linewidth]{JW11_varyb.eps} &
    \psfrag{iteration}[t][]{epoch}
    \includegraphics[width=0.47\linewidth]{MNIST_varyb.eps}
  \end{tabular}
  \caption{Learning curves of different algorithms on tuning sets with different minibatch size $n$.}
  \label{f:varyn}
\end{figure}

\subsection{Effect of minibatch size $n$}
In the first set of experiments, we vary the minibatch size $n$ of NOI over $\{10,20,50,100\}$, while tuning $\rho$, $\eta$ and $\mu$. Learning curves (objective value vs.~number of epochs) on the tuning set for each $n$ with the corresponding optimal hyperparameters are shown in Fig.~\ref{f:varyn}. For comparison, we also show the learning curves of STOL with $n=100$ and $n=500$, while $\eta$ and $\mu$ are also tuned by grid search. We observe that STOL performs very well at $n=500$ (with the performance on MNIST being somewhat better due to higher data redundancy), but it can not achieve much progress in the objective over the random initialization with $n=100$, for the reasons described earlier. In contrast, NOI achieves very competitive performance with various small minibatch sizes, with fast improvement in objective during the first few iterations, although larger $n$ tends to achieve slightly higher correlation on tuning/test sets eventually. Total canonical correlations on the test sets are given in Table~\ref{t:corr}, showing that we achieve better results than \cite{Andrew_13a} with similar DNN architectures.

\subsection{Effect of time constant $\rho$}

\begin{figure}[t]
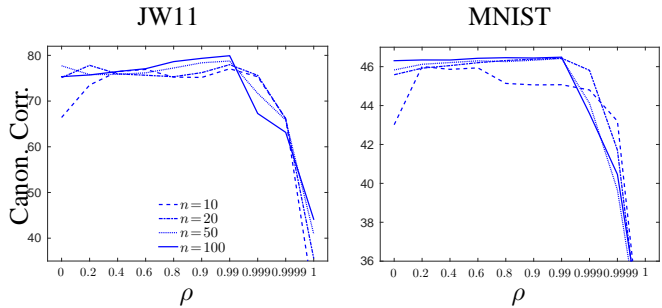

  \centering
    \psfrag{0}[][][.47]{$0$}
    \psfrag{0.2}[][][.47]{$0.2$}
    \psfrag{0.4}[][][.47]{$0.4$}
    \psfrag{0.6}[][][.47]{$0.6$}
    \psfrag{0.8}[][][.47]{$0.8$}
    \psfrag{0.9}[][][.47]{$0.9$}
    \psfrag{0.99}[][][.47]{$0.99$}
    \psfrag{0.999}[][][.47]{$\, 0.999$}
    \psfrag{0.9999}[][][.47]{$\quad 0.9999$}
    \psfrag{1}[][][.47]{$\; 1$}
  \begin{tabular}{@{}c@{\hspace{0.05\linewidth}}c@{}}
    JW11 & MNIST \\[1ex]
    \psfrag{corr}[][]{Canon. Corr.}
    \psfrag{rho}[t][]{$\rho$}
    \psfrag{n=10}[l][l][0.55]{$n\!=\!10$}
    \psfrag{n=20}[l][l][0.55]{$n\!=\!20$}
    \psfrag{n=50}[l][l][0.55]{$n\!=\!50$}
    \psfrag{n=100}[l][l][0.55]{$n\!=\!100$}
    \includegraphics[width=0.49\linewidth]{JW11_varyr.eps} &
    \psfrag{rho}[t][]{$\rho$}
    \includegraphics[width=0.46\linewidth]{MNIST_varyr.eps}
  \end{tabular}
  \caption{Total correlation achieved by NOI on tuning sets with different $\rho$.}
  \label{f:varyr}
\end{figure}

In the second set of experiments, we demonstrate the importance of $\rho$ in NOI for different minibatch sizes. The total canonical correlations achieved by NOI on the tuning set for $\rho=\{0,\, 0.2,\, 0.4,\, 0.6,\, 0.8,\, 0.9,\, 0.99,\, 0.999,\, 0.9999\}$ are shown in Fig.~\ref{f:varyr}, while other hyper-parameters are set to their optimal values. We confirm that for relatively large $n$, NOI works reasonably well with $\rho=0$ (so we are using the same covariance estimate/whitening as \cite{Ma_15b}). But also as expected, when $n$ is small, it is beneficial to incorporate the previous estimate of the covariance because the covariance information contained in each small minibatch is noisy. Also, as $\rho$ becomes too close to $1$, the covariance estimates are not adapted to the DNN outputs and the performance of NOI degrades.  Moreover, we observe that the optimal $\rho$ value seems different for each $n$.

\begin{figure}[t]
  \centering
   \psfrag{0}[][][.7]{$0$}
    \psfrag{0.2}[][][.7]{$0.2$}
    \psfrag{0.4}[][][.7]{$0.4$}
    \psfrag{0.6}[][][.7]{$0.6$}
    \psfrag{0.8}[][][.7]{$0.8$}
    \psfrag{0.9}[][][.7]{$0.9$}
    \psfrag{0.99}[][][.7]{$0.99$}
    \psfrag{0.999}[][][.7]{$\, 0.999$}
    \psfrag{0.9999}[][][.7]{$\quad 0.9999$}
  \psfrag{corr}[][]{Canon. Corr.}
  \psfrag{rho}[][]{$\rho$}
  \psfrag{Initialization}[l][l][0.8]{Random Init.}
  \psfrag{SVD}[l][l][0.8]{SVD}
  \psfrag{STOL n=500}[l][l][0.8]{STOL $n\!=\!500$}
  \psfrag{NOI n=1}[l][l][0.8]{NOI $n\!=\!1$}
  \includegraphics[width=0.8\linewidth]{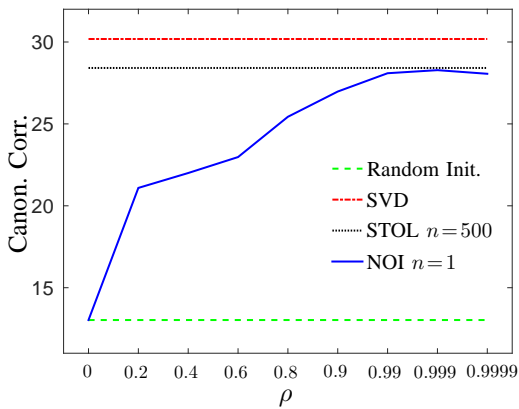}
  \caption{Pure stochastic optimization of linear CCA using NOI. We show total correlation achieved by NOI with $n=1$ on the MNIST training sets at different $\rho$, by the random initialization used by NOI, by the exact solution, and by STOL with $n=500$.}
  \label{f:cca-noi}
\end{figure} 

\subsection{Pure stochastic optimization for CCA}

Finally, we carry out pure stochastic optimization ($n=1$) for linear CCA on the MNIST dataset. Notice that linear CCA is a special case of DCCA with $(\tf,\tg)$ both being single-layer linear networks (although we have used small weight-decay terms for the weights, leading to a slightly different objective than that of CCA). Total canonical correlations achieved by STOL with $n=500$ and by NOI (50 training epochs) on the training set with different $\rho$ values are shown in Fig.~\ref{f:cca-noi}. The objective of the random initialization and the closed-form solution (by SVD) are also shown for comparison. NOI could not improve over the random initialization without memory ($\rho=0$, corresponding to the algorithm of \cite{Ma_15b}), but gets very close to the optimal solution and matches the objective obtained by the previous large minibatch approach when $\rho\rightarrow 1$. This result demonstrates the importance of our adaptive estimate \eqref{e:memory} also for CCA.

\section{Conclusions}
\label{s:conclusion}

In this paper, we have proposed a stochastic optimization algorithm NOI for training DCCA which updates the DNN weights based on small minibatches and performs competitively to previous optimizers.

One direction for future work is to better understand the convergence properties of NOI, which presents several difficulties. First, we note that convergence of the alternating least squares formulation of CCA (Algorithm~\ref{alg:cca-iterative}, or rather orthogonal iterations) is usually stated as the angle between the estimated subspace and the ground-truth subspace converging to zero. In the stochastic optimization setting, we need to relate this measure of progress (or some other measure) to the nonlinear least squares problems we are trying to solve in the NOI iterations. As discussed in Section~\ref{s:related}, even the convergence 
of the linear CCA version of NOI with batch gradient descent is not well understood~\cite{Ma_15b}. Second, the use of memory in estimating covariances \eqref{e:memory} complicates the analysis and ideally we would like to come up with ways of determining the time constant $\rho$. 

We have also tried using the same form of adaptive covariance estimates in both views for the STOL approach for computing the gradients \eqref{e:gradient}, but its performance with small minibatches is much worse than that of NOI. Presumably this is because the gradient computation of STOL suffers from noise in both views which are further combined through various nonlinear operations, whereas the noise in the gradient computation of NOI only comes from the output target (due to inexact whitening), and as a result NOI is more tolerant to the noise resulting from using small minibatches.  This deserves further analysis as well.

\bibliographystyle{IEEEtran}
\bibliography{allerton15a}

\end{document}